\documentclass{article}
\usepackage{fullpage,hyperref}
\usepackage{algorithm,algorithmic}
\usepackage{amsmath,amsfonts,amsthm,dsfont,color,graphicx}

\newtheorem{lemma}{Lemma}
\newtheorem{theorem}{Theorem}

\def\E{\mathbb{E}}
\def\D{\mathcal{D}}
\def\H{\mathcal{H}}
\def\X{\mathcal{X}}
\def\Y{\mathcal{Y}}
\def\R{\mathbb{R}}
\def\DIS{\mathrm{DIS}}
\def\I{\mathds{1}}
\def\wh{\widehat}

\def\veps{\varepsilon}

\DeclareMathOperator{\err}{err}

\newcommand\independent{\protect\mathpalette{\protect\independenT}{\perp}}
\def\independenT#1#2{\mathrel{\rlap{$#1#2$}\mkern2mu{#1#2}}}

\title{Agnostic Active Learning Without Constraints}

\author{
Alina Beygelzimer\footnote{IBM Research, \texttt{beygel@us.ibm.com}}
\and
Daniel Hsu\footnote{UC San Diego, \texttt{djhsu@cs.ucsd.edu}}
\and
John Langford\footnote{Yahoo!~Research, \texttt{jl@yahoo-inc.com}}
\and
Tong Zhang\footnote{Rutgers University, \texttt{tongz@rci.rutgers.edu}}
}

\begin{document}

\maketitle

\begin{abstract}
We present and analyze an agnostic active learning algorithm that works
without keeping a version space.
This is unlike all previous approaches where a restricted set of candidate
hypotheses is maintained throughout learning, and only hypotheses from this
set are ever returned.
By avoiding this version space approach, our algorithm sheds the
computational burden and brittleness associated with maintaining version
spaces, yet still allows for substantial improvements over supervised
learning for classification.
\end{abstract}

\section{Introduction}
 
In active learning, a learner is given access to unlabeled data and is
allowed to adaptively choose which ones to label.
This learning model is motivated by applications in which the cost of
labeling data is high relative to that of collecting the unlabeled data
itself.
Therefore, the hope is that the active learner only needs to query the
labels of a small number of the unlabeled data, and otherwise perform as
well as a fully supervised learner.
In this work, we are interested in agnostic active learning algorithms for
binary classification that are provably consistent, \emph{i.e.} that
converge to an optimal hypothesis in a given hypothesis class.

One technique that has proved theoretically profitable is to maintain a
candidate set of hypotheses (sometimes called a version space), and to
query the label of a point only if there is disagreement within this set
about how to label the point.
The criteria for membership in this candidate set needs to be carefully
defined so that an optimal hypothesis is always included, but otherwise
this set can be quickly whittled down as more labels are queried.
This technique is perhaps most readily understood in the noise-free setting
\cite{CAL94,Das05}, and it can be extended to noisy settings by using
empirical confidence bounds~\cite{BBL06,DHM07,BDL09,Han09,Kol09}.

The version space approach unfortunately has its share of significant
drawbacks.
The first is computational intractability:
maintaining a version space and guaranteeing that \emph{only}
hypotheses from this set are returned is difficult for linear
predictors and appears intractable for interesting nonlinear
predictors such as neural nets and decision trees~\cite{CAL94}.
Another drawback of the approach is its brittleness: a single mishap (due
to, say, modeling failures or computational approximations) might cause the
learner to exclude the best hypothesis from the version space forever; this
is an ungraceful failure mode that is not easy to correct.
A third drawback is related to sample re-usability: if (labeled) data is
collected using a version space-based active learning algorithm, and we
later decide to use a different algorithm or hypothesis class, then the
earlier data may not be freely re-used because its collection process is
inherently biased.

Here, we develop a new strategy addressing all of the above problems
given an oracle that returns an empirical risk minimizing (ERM)
hypothesis.
As this oracle matches our abstraction of many supervised learning
algorithms, we believe active learning algorithms built in this way are
immediately and widely applicable.

Our approach instantiates the importance weighted active learning
framework of \cite{BDL09} using a rejection threshold similar to the
algorithm of \cite{DHM07} which only accesses hypotheses via a
supervised learning oracle.  However, the oracle we require is simpler
and avoids strict adherence to a candidate set of hypotheses.
Moreover, our algorithm creates an importance weighted sample that
allows for unbiased risk estimation, even for hypotheses from a class
different from the one employed by the active learner.  This is in
sharp contrast to many previous algorithms (\emph{e.g.},
\cite{CAL94,BBL06,BBZ07,DHM07,Han09,Kol09}) that create heavily biased
data sets.
We prove that our algorithm is always consistent and has an improved label
complexity over passive learning in cases previously studied in the
literature.
We also describe a practical instantiation of our algorithm and report on
some experimental results.

\subsection{Related Work}

As already mentioned, our work is closely related to the previous works of
\cite{DHM07} and \cite{BDL09}, both of which in turn draw heavily on the
work of \cite{CAL94} and \cite{BBL06}.
The algorithm from~\cite{DHM07} extends the selective sampling method of
\cite{CAL94} to the agnostic setting using generalization bounds in a
manner similar to that first suggested in~\cite{BBL06}.
It accesses hypotheses only through a special ERM oracle that can enforce
an arbitrary number of example-based constraints; these constraints define
a version space, and the algorithm only ever returns hypotheses from this
space, which can be undesirable as we previously argued.
Other previous algorithms with comparable performance guarantees also
require similar example-based constraints (\emph{e.g.},
\cite{BBL06,BDL09,Han09,Kol09}).
Our algorithm differs from these in that (i) it never restricts its
attention to a version space when selecting a hypothesis to return, and
(ii) it only requires an ERM oracle that enforces at most one example-based
constraint, and this constraint is only used for selective sampling.
Our label complexity bounds are comparable to those proved in~\cite{BDL09}
(though somewhat worse that those in~\cite{BBL06,DHM07,Han09,Kol09}).

The use of importance weights to correct for sampling bias is a standard
technique for many machine learning problems
(\emph{e.g.},~\cite{SB98,ACBFS02,SKM07}) including active
learning~\cite{Sug05,Bac06,BDL09}.
Our algorithm is based on the importance weighted active learning (IWAL)
framework introduced by \cite{BDL09}.
In that work, a rejection threshold procedure called \emph{loss-weighting}
is rigorously analyzed and shown to yield improved label complexity bounds
in certain cases.
Loss-weighting is more general than our technique in that it extends beyond
zero-one loss to a certain subclass of loss functions such as logistic
loss.
On the other hand, the loss-weighting rejection threshold requires
optimizing over a restricted version space, which is computationally
undesirable.
Moreover, the label complexity bound given in~\cite{BDL09} only applies
to hypotheses selected from this version space, and not when selected from
the entire hypothesis class (as the general IWAL framework suggests).
We avoid these deficiencies using a new rejection threshold procedure and a
more subtle martingale analysis.

Many of the previously mentioned algorithms are analyzed in the agnostic
learning model, where no assumption is made about the noise distribution
(see also~\cite{Han07}).
In this setting, the label complexity of active learning algorithms cannot
generally improve over supervised learners by more than a constant
factor~\cite{Kaa06,BDL09}.
However, under a parameterization of the noise distribution related to
Tsybakov's low-noise condition~\cite{Tsy04}, active learning algorithms
have been shown to have improved label complexity bounds over what is
achievable in the purely agnostic
setting~\cite{CN06,BBZ07,CN07,Han09,Kol09}.
We also consider this parameterization to obtain a tighter label complexity
analysis.

\section{Preliminaries}

\subsection{Learning Model}

Let $\D$ be a distribution over $\X \times \Y$
where $\X$ is the input space and $\Y = \{\pm1\}$ are the labels.
Let $(X,Y) \in \X\times\Y$ be a pair of random variables with joint
distribution $\D$.
An active learner receives a sequence $(X_1,Y_1), (X_2,Y_2), \ldots$ of
i.i.d.~copies of $(X,Y)$, with the label $Y_i$ hidden unless it is
explicitly queried.
We use the shorthand $a_{1:k}$ to denote a sequence $(a_1,a_2,\ldots,a_k)$
(so $k=0$ correspond to the empty sequence).

Let $\H$ be a set of hypotheses mapping from $\X$ to $\Y$.
For simplicity, we assume $\H$ is finite but does not completely agree on
any single $x \in \X$ (\emph{i.e.}, $\forall x \in \X, \exists h,h' \in \H$
such that $h(x) \neq h'(x)$).
This keeps the focus on the relevant aspects of active learning that differ
from passive learning.
The error of a hypothesis $h:\X\to\Y$ is $\err(h) := \Pr(h(X) \neq Y)$.
Let $h^* := \arg\min \{ \err(h) : h \in \H \}$ be a hypothesis of minimum
error in $\H$.
The goal of the active learner is to return a hypothesis $h \in \H$ with
error $\err(h)$ not much more than $\err(h^*)$, using as few label queries
as possible.

\subsection{Importance Weighted Active Learning} \label{section:iwal}

In the importance weighted active learning (IWAL) framework of
\cite{BDL09}, an active learner looks at the unlabeled data $X_1, X_2,
\ldots$ one at a time.
After each new point $X_i$, the learner determines a probability $P_i \in
[0,1]$.
Then a coin with bias $P_i$ is flipped, and the label $Y_i$ is queried if
and only if the coin comes up heads.
The query probability $P_i$ can depend on all previous unlabeled examples
$X_{1:i-1}$, any previously queried labels, any past coin flips, and the
current unlabeled point $X_i$.

Formally, an IWAL algorithm specifies a \emph{rejection threshold} function
$p:(\X\times\Y\times\{0,1\})^* \times \X \to [0,1]$ for determining these
query probabilities.
Let $Q_i \in \{0,1\}$ be a random variable conditionally independent of the
current label $Y_i$,
\begin{equation*}
Q_i \independent Y_i \ | \ X_{1:i},Y_{1:i-1},Q_{1:i-1}
\end{equation*}
and with conditional expectation
\begin{equation*}
\E[Q_i|Z_{1:i-1},X_i]
\ = \
P_i
\ := \
p(Z_{1:i-1},X_i)
.
\end{equation*}
where $Z_j := (X_j,Y_j,Q_j)$.
That is, $Q_i$ indicates if the label $Y_i$ is queried (the outcome of the
coin toss).
Although the notation does not explicitly suggest this, the query
probability $P_i = p(Z_{1:i-1},X_i)$ is allowed to explicitly depend on a
label $Y_j$ ($j < i$) if and only if it has been queried ($Q_j = 1$).

\subsection{Importance Weighted Estimators}

We first review some standard facts about the importance weighting
technique.
For a function $f:\X\times\Y\to\R$, define the \emph{importance weighted
estimator} of $\E[f(X,Y)]$ from $Z_{1:n} \in (\X \times \Y \times
\{0,1\})^n$ to be
\begin{equation*}
\wh f(Z_{1:n})
\ := \
\frac1n \sum_{i=1}^n \frac{Q_i}{P_i} \cdot f(X_i,Y_i)
.
\end{equation*}
Note that this quantity depends on a label $Y_i$ only if it has been
queried (\emph{i.e.}, only if $Q_i = 1$; it also depends on $X_i$ only if
$Q_i = 1$).
Our rejection threshold will be based on a specialization of this
estimator, specifically the \emph{importance weighted empirical error} of a
hypothesis $h$
\begin{equation*}
\err(h,Z_{1:n}) \ := \
\frac1n \sum_{i=1}^n \frac{Q_i}{P_i} \cdot \I[h(X_i) \neq Y_i]
.
\end{equation*}
In the notation of Algorithm 1, this is equivalent to
\begin{equation} \label{eq:iw-err}
\err(h,S_n) \ := \
\frac1n \sum_{(X_i,Y_i,1/P_i) \in S_n} (1/P_i) \cdot \I[h(X_i) \neq Y_i]
\end{equation}
where $S_n \subseteq \X\times\Y\times\R$ is the importance weighted
sample collected by the algorithm.

A basic property of these estimators is \emph{unbiasedness}:
\begin{align*}
\E[\wh f(Z_{1:n})]
& = \frac1n \sum_{i=1}^n
\E[\E[ (Q_i/P_i) \cdot f(X_i,Y_i) \ | \ X_{1:i},Y_{1:i},Q_{1:i-1} ] ]
\\
& = \frac1n \sum_{i=1}^n
\E[ (P_i/P_i) \cdot f(X_i,Y_i) ]
\\
& = \E[f(X,Y)]
.
\end{align*}
So, for example, the importance weighted empirical error of a hypothesis
$h$ is an unbiased estimator of its true error $\err(h)$.
This holds for \emph{any} choice of the rejection threshold that guarantees
$P_i > 0$.

\section{A Deviation Bound for Importance Weighted Estimators}

As mentioned before, the rejection threshold used by our algorithm is based
on importance weighted error estimates $\err(h,Z_{1:n})$.
Even though these estimates are unbiased, they are only reliable when the
variance is not too large.
To get a handle on this, we need a deviation bound for importance weighted
estimators.
This is complicated by two factors that rules out straightforward
applications of some standard bounds:
\begin{enumerate}
\item The importance weighted samples $(X_i,Y_i,1/P_i)$ (or equivalently,
the $Z_i = (X_i,Y_i,Q_i)$) are not i.i.d.
This is because the query probability $P_i$ (and thus the importance weight
$1/P_i$) generally depends on $Z_{1:i-1}$ and $X_i$.

\item The effective range and variance of each term in the estimator are,
themselves, random variables.

\end{enumerate}
To address these issues, we develop a deviation bound using a martingale
technique from~\cite{Zha05}.

Let $f:\X\times\Y\to[-1,1]$ be a bounded function.
Consider any rejection threshold function $p:(\X\times\Y\times\{0,1\})^*
\times \X \to (0,1]$ for which $P_n = p(Z_{1:n-1},X_n)$ is bounded below by
some positive quantity (which may depend on $n$).
Equivalently, the query probabilities $P_n$ should have inverses $1/P_n$
bounded above by some deterministic quantity $r_{max}$ (which, again, may
depend on $n$).
The \emph{a priori} upper bound $r_{max}$ on $1/P_n$ can be pessimistic, as
the dependence on $r_{max}$ in the final deviation bound will be very
mild---it enters in as $\log\log r_{max}$.
Our goal is to prove a bound on $|\wh f(Z_{1:n}) - \E[f(X,Y)]|$ that holds
with high probability over the joint distribution of $Z_{1:n}$.

To start, we establish bounds on the range and variance of each term $W_i
:= (Q_i/P_i) \cdot f(X_i,Y_i)$ in the estimator, conditioned on
$(X_{1:i},Y_{1:i},Q_{1:i-1})$.
Let $\E_i[\ \cdot\ ]$ denote $\E[\ \cdot\ |X_{1:i},Y_{1:i},Q_{1:i-1}]$.
Note that $\E_i[W_i] = (\E_i[Q_i]/P_i) \cdot f(X_i,Y_i) = f(X_i,Y_i)$,
so if $\E_i[W_i] = 0$, then $W_i = 0$.
Therefore, the (conditional) range and variance are non-zero only if
$\E_i[W_i] \neq 0$.
For the range, we have
$|W_i| = (Q_i/P_i) \cdot |f(X_i,Y_i)| \leq 1/P_i$,
and for the variance,
$\E_i[(W_i-\E_i[W_i])^2]
\leq (\E_i[Q_i^2]/P_i^2) \cdot f(X_i,Y_i)^2
\leq 1/P_i$.
These range and variance bounds indicate the form of the deviations we can
expect, similar to that of other classical deviation bounds.

\begin{theorem} \label{theorem:iw-deviation}
Pick any $t \geq 0$ and $n \geq 1$.
Assume $1 \leq 1/P_i \leq r_{max}$ for all $1 \leq i \leq n$, and let $R_n
:= 1 / \min(\{P_i : 1 \leq i \leq n \ \wedge \ f(X_i,Y_i) \neq 0 \} \cup \{
1 \})$.
With probability at least $1 - 2(3+\log_2 r_{max}) e^{-t/2}$,
\begin{equation*}
\left|
\frac1n \sum_{i=1}^n \frac{Q_i}{P_i} \cdot f(X_i,Y_i)
-
\E[f(X,Y)]
\right|
\ \leq \
\sqrt{\frac{2R_nt}{n}}
+ \sqrt{\frac{2t}{n}}
+ \frac{R_nt}{3n}
.
\end{equation*}
\end{theorem}
We defer all proofs to the appendices.

\section{Algorithm} \label{section:alg}

\begin{figure}
\framebox[\textwidth]{\begin{minipage}{0.95\textwidth}
\textbf{Algorithm 1} \\
Notes: see Eq.~\eqref{eq:iw-err} for the definition of $\err$ (importance
weighted error), and Section~\ref{section:alg} for the definitions of
$C_0$, $c_1$, and $c_2$.
\\
Initialize: $S_0 := \emptyset$. \\
For $k = 1, 2, \ldots, n$:
\begin{list}{}{\setlength{\leftmargin}{2em}}
\item[1.] Obtain unlabeled data point $X_k$.

\item[2.] Let
\begin{list}{}{\setlength{\leftmargin}{1em}}
\item $h_k := \arg \min \{ \err(h, S_{k-1}) : h \in \H \}$, and
\item $h_k' := \arg \min \{ \err(h, S_{k-1}) : h \in \H \ \wedge \ h(X_k)
\neq h_k(X_k) \}$.
\end{list}

Let $G_k := \err(h_k', S_{k-1}) - \err(h_k, S_{k-1})$, and
\[ P_k :=
\left\{\begin{array}{ll}
1 & \text{if
$G_k \leq \sqrt{\frac{C_0 \log k}{k-1}} + \frac{C_0 \log k}{k-1}$
} \\
s & \text{otherwise}
\end{array}\right.
\left(
 =
\min\left\{ 1, \ O\left( \frac{1}{G_k^2} + \frac1{G_k} \right)\cdot \frac{C_0 \log k}{k-1}
\right\}
\right)
\]
where $s \in (0,1)$ is the positive solution to the equation
\begin{equation} \label{eq:ical-quadratic}
G_k = \left( \frac{c_1}{\sqrt{s}} - c_1 + 1 \right) \cdot
\sqrt{\frac{C_0 \log k}{k-1}}
+ \left( \frac{c_2}{s} - c_2 + 1 \right) \cdot \frac{C_0 \log k}{k-1}
.
\end{equation}

\item[3.] Toss a biased coin with $\Pr(\text{heads}) = P_k$.
\begin{list}{}{\setlength{\leftmargin}{1em}}
\item If heads, then query $Y_k$, and let $S_k := S_{k-1} \cup \{
(X_k,Y_k,1/P_k) \}$.

\item Else, let $S_k := S_{k-1}$.

\end{list}

\end{list}
Return: $h_{n+1} := \arg\min\{\err(h,S_n) : h \in \H\}$.

\end{minipage}}
\caption{Algorithm for importance weighted active learning with an error
minimization oracle.}
\label{fig:ical}
\end{figure}

First, we state a deviation bound for the importance weighted error of
hypotheses in a finite hypothesis class $\H$ that holds for all $n \geq 1$.
It is a simple consequence of Theorem~\ref{theorem:iw-deviation} and union
bounds; the form of the bound motivates certain algorithmic choices to be
described below.
\begin{lemma} \label{lemma:ical-deviation}
Pick any $\delta \in (0,1)$.
For all $n \geq 1$, let
\begin{equation} \label{eq:ical-eps}
\veps_n := \frac{16\log(2(3+n\log_2 n)n(n+1)|\H|/\delta)}{n}
= O\left( \frac{\log(n|\H|/\delta)}{n} \right)
.
\end{equation}
Let $(Z_1, Z_2, \ldots) \in (\X\times\Y\times\{0,1\})^*$ be the sequence of
random variables specified in Section~\ref{section:iwal} using a rejection
threshold $p:(\X\times\Y\times\{0,1\})^* \times \X \to [0,1]$ that
satisfies $p(z_{1:n},x) \geq 1/n^n$ for all $(z_{1:n},x) \in
(\X\times\Y\times\{0,1\})^n \times \X$ and all $n \geq 1$.

The following holds with probability at least $1-\delta$.
For all $n \geq 1$ and all $h \in \H$,
\begin{gather}
|(\err(h,Z_{1:n}) - \err(h^*,Z_{1:n})) - (\err(h) - \err(h^*))|
\leq \sqrt{\frac{\veps_n}{P_{min,n}(h)}}
+ \frac{\veps_n}{P_{min,n}(h)}
\label{eq:ical-deviation}
\end{gather}
where
$P_{min,n}(h)
= \min\{ P_i : 1 \leq i \leq n \ \wedge \ h(X_i) \neq h^*(X_i) \} \cup
\{1\}$
.
\end{lemma}
We let $C_0 = O(\log(|\H|/\delta)) \geq 2$ be a quantity such that
$\veps_n$ (as defined in Eq.~\eqref{eq:ical-eps}) is bounded as $\veps_n
\leq C_0 \cdot \log (n+1)/n$.
The following absolute constants are used in the description of the
rejection threshold and the subsequent analysis:
$c_1 := 5 + 2\sqrt2$,
$c_2 := 5$,
$c_3 := ((c_1+\sqrt2)/(c_1-2))^2$,
$c_4 := (c_1 + \sqrt{c_3})^2$,
$c_5 := c_2 + c_3$
.

Our proposed algorithm is shown in Figure~\ref{fig:ical}.
The rejection threshold (Step 2) is based on the deviation bound from
Lemma~\ref{lemma:ical-deviation}.
First, the importance weighted error minimizing hypothesis $h_k$ and the
``alternative'' hypothesis $h_k'$ are found.
Note that both optimizations are over the entire hypothesis class $\H$
(with $h_k'$ only being required to disagree with $h_k$ on $x_k$)---this
is a key aspect where our algorithm differs from previous approaches.
The difference in importance weighted errors $G_k$ of the two hypotheses is
then computed.
If $G_k \leq \sqrt{(C_0 \log k)/(k-1)} + (C_0 \log k)/(k-1)$, then the
query probability $P_k$ is set to $1$.
Otherwise, $P_k$ is set to the positive solution $s$ to the quadratic
equation in Eq.~\eqref{eq:ical-quadratic}.
The functional form of $P_k$ is roughly
\[
\min\left\{ 1, \ O\left(\frac1{G_k^2} + \frac1{G_k}\right) \cdot \frac{C_0
\log k}{k-1} \right\}
. \]
It can be checked that $P_k \in (0,1]$ and that $P_k$ is non-increasing
with $G_k$.
It is also useful to note that $(\log k)/(k-1)$ is monotonically decreasing
with $k \geq 1$ (we use the convention $\log(1)/0 = \infty$).

In order to apply Lemma~\ref{lemma:ical-deviation} with our rejection
threshold, we need to establish the (very crude) bound $P_k \geq 1/k^k$ for
all $k$.
\begin{lemma} \label{lemma:ical-crude-bound}
The rejection threshold of Algorithm 1 satisfies $p(z_{1:n-1},x) \geq
1/n^n$ for all $n \geq 1$ and all $(z_{1:n-1},x) \in
(\X\times\Y\times\{0,1\})^{n-1} \times \X$.
\end{lemma}
Note that this is a worst-case bound; our analysis shows that the
probabilities $P_k$ are more like $1/\mbox{poly}(k)$ in the typical case.

\section{Analysis}

\subsection{Correctness}

We first prove a consistency guarantee for Algorithm 1 that bounds the
generalization error of the importance weighted empirical error minimizer.
The proof actually establishes a lower bound on the query probabilities
$P_i \geq 1/2$ for $X_i$ such that $h_n(X_i) \neq h^*(X_i)$.
This offers an intuitive characterization of the weighting landscape
induced by the importance weights $1/P_i$.
\begin{theorem} \label{theorem:ical-consistency}
The following holds with probability at least $1 - \delta$.
For any $n \geq 1$,
\begin{equation*}
0
\ \leq \
\err(h_n) - \err(h^*)
\ \leq \
\err(h_n,Z_{1:n-1}) - \err(h^*,Z_{1:n-1}) + \sqrt{\frac{2C_0 \log n}{n-1}}
+ \frac{2C_0 \log n}{n-1}
.
\end{equation*}
This implies, for all $n \geq 1$,
\begin{equation*}
\err(h_n) \ \leq \ \err(h^*)
+ \sqrt{\frac{2C_0 \log n}{n-1}} + \frac{2C_0 \log n}{n-1}
.
\end{equation*}
\end{theorem}
Therefore, the final hypothesis returned by Algorithm 1 after seeing $n$
unlabeled data has roughly the same error bound as a hypothesis returned by
a standard passive learner with $n$ labeled data.
A variant of this result under certain noise conditions is given in the
appendix.

\subsection{Label Complexity Analysis}

We now bound the number of labels requested by Algorithm 1 after $n$
iterations.
The following lemma bounds the probability of querying the label $Y_n$;
this is subsequently used to establish the final bound on the expected
number of labels queried.
The key to the proof is in relating empirical error differences and their
deviations to the probability of querying a label.
This is mediated through the \emph{disagreement coefficient}, a quantity
first used by \cite{Han07} for analyzing the label complexity of the $A^2$
algorithm of~\cite{BBL06}.
The disagreement coefficient $\theta := \theta(h^*,\H,\D)$ is defined as
\[ \theta(h^*,\H,\D) := \sup \left\{ \frac{\Pr(X \in \DIS(h^*,r))}{r} : r >
0 \right\}
\]
where
\[ \DIS(h^*,r) := \{ x \in \X : \exists h' \in \H \ \text{such that} \
\Pr(h^*(X) \neq h'(X)) \leq r \ \text{and} \ h^*(x) \neq h'(x) \}
\]
(the disagreement region around $h^*$ at radius $r$).
This quantity is bounded for many learning problems studied in the
literature; see \cite{Han07,Han09,Fri09,Wan09} for more discussion.
Note that the supremum can instead be taken over $r > \epsilon$ if the
target excess error is $\epsilon$, which allows for a more detailed
analysis.

\begin{lemma} \label{lemma:ical-query}
Assume the bounds from Eq.~\eqref{eq:ical-deviation} holds for all $h \in
\H$ and $n \geq 1$.
For any $n \geq 1$,
\begin{equation*}
\E[Q_n]
\leq
\theta \cdot 2\err(h^*)
+ O\left( \theta \cdot \sqrt{\frac{C_0 \log n}{n-1}}
+ \theta \cdot \frac{C_0 \log^2n}{n-1} \right)
.
\end{equation*}

\end{lemma}

\begin{theorem} \label{theorem:ical-labelcomplexity}
With probability at least $1-\delta$,
the expected number of labels queried by Algorithm 1 after $n$ iterations
is at most
\begin{equation*}
1 + \theta \cdot 2\err(h^*) \cdot (n-1)
+ O\left( \theta \cdot \sqrt{C_0 n \log n}
+ \theta \cdot C_0 \log^3 n
\right)
.
\end{equation*}
\end{theorem}
\begin{proof}
Follows from assuming $Y_1$ is always queried; applying
Lemmas~\ref{lemma:ical-deviation}, \ref{lemma:ical-crude-bound},
\ref{lemma:ical-query}, and linearity of expectation.
\end{proof}
The bound is dominated by a linear term scaled by $\err(h^*)$, plus a
sublinear term.
The linear term $\err(h^*) \cdot n$ is unavoidable in the worst case, as
evident from label complexity lower bounds~\cite{Kaa06,BDL09}.
When $\err(h^*)$ is negligible (\emph{e.g.}, the data is separable) and
$\theta$ is bounded (as is the case for many problems studied in the
literature~\cite{Han07}), then the bound represents a polynomial
label complexity improvement over supervised learning, similar to that
achieved by the version space algorithm from~\cite{BDL09}.

\subsection{Analysis under Low Noise Conditions}

Some recent work on active learning has focused on improved label
complexity under certain noise conditions~\cite{CN06,BBZ07,CN07,Han09,Kol09}.
Specifically, it is assumed that there exists constants $\kappa > 0$ and
$0 < \alpha \leq 1$ such that
\begin{equation} \label{eq:tsybakov}
\Pr(h(X) \neq h^*(X)) \leq \kappa \cdot \left(\err(h) -
\err(h^*)\right)^{\alpha}
\end{equation}
for all $h \in \H$.
This is related to Tsybakov's low noise condition~\cite{Tsy04}.
Essentially, this condition requires that low error hypotheses not be too
far from the optimal hypothesis $h^*$ under the disagreement metric
$\Pr(h^*(X) \neq h(X))$.
Under this condition, Lemma~\ref{lemma:ical-query} can be improved, which
in turn yields the following theorem.
\begin{theorem} \label{theorem:ical-labelcomplexity2}
Assume that for some value of $\kappa > 0$ and $0 < \alpha \leq 1$, the
condition in Eq.~\eqref{eq:tsybakov} holds for all $h \in \H$.
There is a constant $c_\alpha > 0$ depending only on $\alpha$ such that the
following holds.
With probability at least $1-\delta$,
the expected number of labels queried by Algorithm 1 after $n$ iterations
is at most
\begin{equation*}
\theta \cdot \kappa \cdot c_\alpha
\cdot \left( C_0 \log n\right)^{\alpha/2}
\cdot n^{1-\alpha/2}
.
\end{equation*}
\end{theorem}
Note that the bound is sublinear in $n$ for all $0 < \alpha \leq 1$, which
implies label complexity improvements whenever $\theta$ is bounded (an
improved analogue of Theorem~\ref{theorem:ical-consistency} under these
conditions can be established using similar techniques).
The previous algorithms of~\cite{Han09,Kol09} obtain even better rates
under these noise conditions using specialized data dependent
generalization bounds, but these algorithms also required optimizations
over restricted version spaces, even for the bound computation.

\section{Experiments}

Although agnostic learning is typically intractable in the worst case,
empirical risk minimization can serve as a useful abstraction for many
practical supervised learning algorithms in non-worst case scenarios.
With this in mind, we conducted a preliminary experimental evaluation
of Algorithm 1, implemented using a popular algorithm for learning
decision trees in place of the required ERM oracle.  Specifically, we
use the \texttt{J48} algorithm from Weka v3.6.2 (with default
parameters) to select the hypothesis $h_k$ in each round $k$; to
produce the ``alternative'' hypothesis $h_k'$, we just modify the
decision tree $h_k$ by changing the label of the node used for
predicting on $x_k$.  Both of these procedures are clearly heuristic,
but they are similar in spirit to the required optimizations.  We set
$C_0 = 8$ and $c_1 = c_2 = 1$---these can be regarded as tuning
parameters, with $C_0$ controlling the aggressiveness of the rejection
threshold.  We did not perform parameter tuning with active learning
although the importance weighting approach developed here could
potentially be used for that.  Rather, the goal of these experiments is
to assess the compatibility of Algorithm 1 with an existing, practical
supervised learning procedure.

\subsection{Data Sets}

We constructed two binary classification tasks using MNIST and KDDCUP99
data sets.
For MNIST, we randomly chose $4000$ training $3$s and $5$s for training
(using the $3$s as the positive class), and used all of the $1902$ testing
$3$s and $5$s for testing.
For KDDCUP99, we randomly chose $5000$ examples for training, and another
$5000$ for testing.
In both cases, we reduced the dimension of the data to $25$ using PCA.

To demonstrate the versatility of our algorithm, we also conducted a
multi-class classification experiment using the entire MNIST data set (all
ten digits, so $60000$ training data and $10000$ testing data).
This required modifying how $h_k'$ is selected: we force $h_k'(x_k) \neq
h_k(x_k)$ by changing the label of the prediction node for $x_k$ to the
next best label.
We used PCA to reduce the dimension to $40$.

\subsection{Results}

\begin{figure}
\begin{center}
\begin{tabular}{cc}
\includegraphics[height=0.2\textheight]{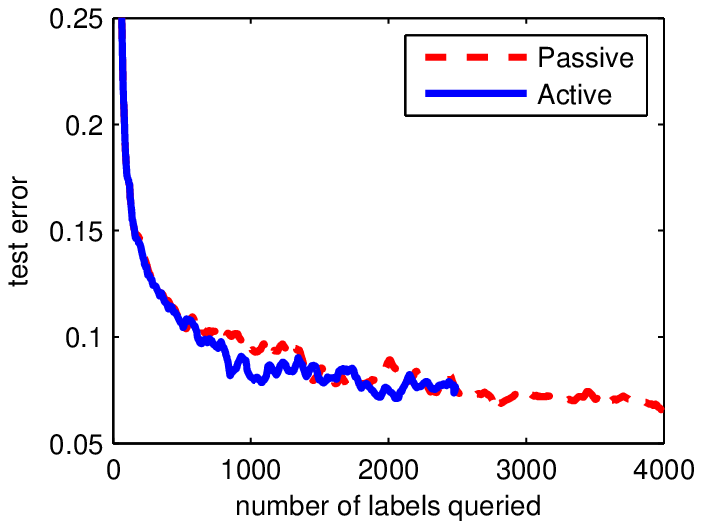}
&
\includegraphics[height=0.2\textheight]{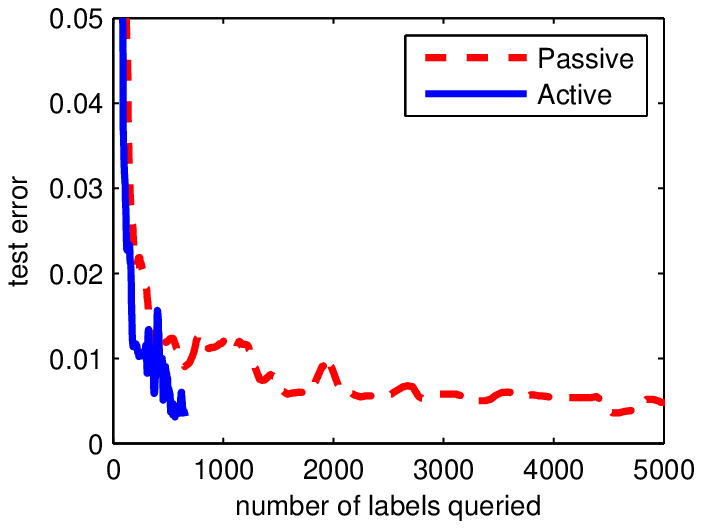}
\\
MNIST $3$s vs $5$s &
KDDCUP99
\\
\includegraphics[height=0.2\textheight]{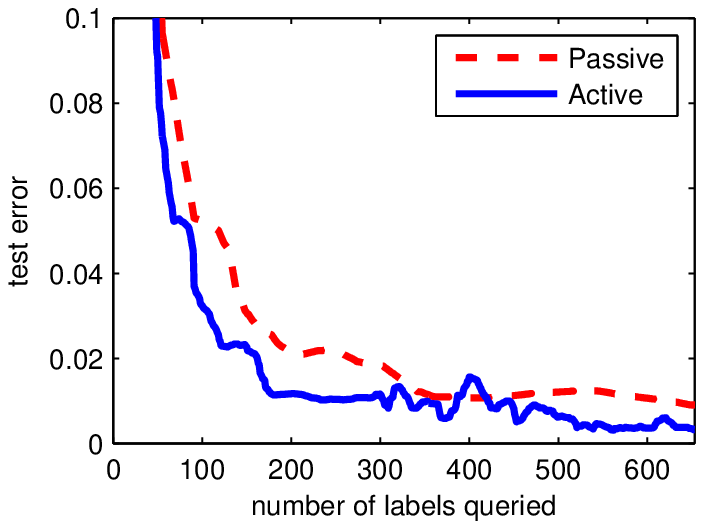}
&
\includegraphics[height=0.2\textheight]{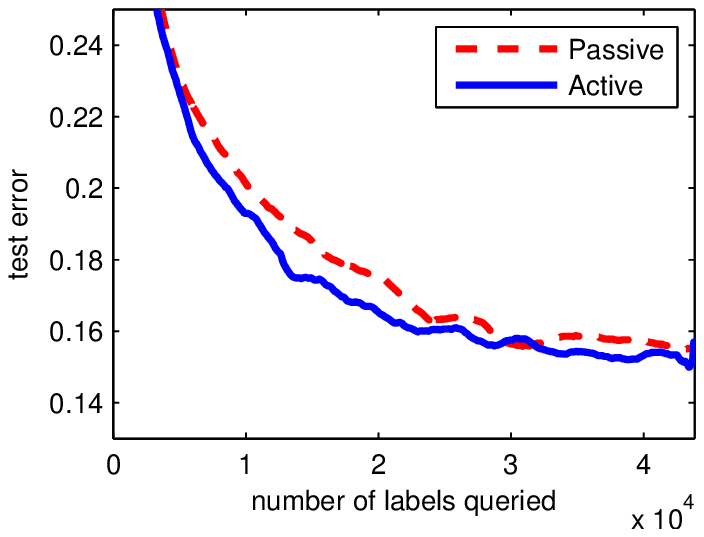}
\\
KDDCUP99 (close-up)
&
MNIST multi-class (close-up)
\end{tabular}
\end{center}
\caption{Test errors as a function of the number of labels queried.}
\label{fig:test-error-labeled}
\end{figure}

We examined the test error as a function of (i) the number of unlabeled
data seen, and (ii) the number of labels queried.
We compared the performance of the active learner described above to a
passive learner (one that queries every label, so (i) and (ii) are the
same) using \texttt{J48} with default parameters.

In all three cases, the test errors as a function of the number of
unlabeled data were roughly the same for both the active and passive
learners.
This agrees with the consistency guarantee from
Theorem~\ref{theorem:ical-consistency}.
We note that this is a basic property \emph{not} satisfied by many active
learning algorithms (this issue is discussed further in~\cite{DH08}).

In terms of test error as a function of the number of labels queried
(Figure~\ref{fig:test-error-labeled}), the active learner had minimal
improvement over the passive learner on the binary MNIST task, but a
substantial improvement over the passive learner on the KDDCUP99 task
(even at small numbers of label queries).  For the multi-class MNIST
task, the active learner had a moderate improvement over the passive
learner.  Note that KDDCUP99 is far less noisy (more separable) than
MNIST $3$s vs $5$s task, so the results are in line with the label
complexity behavior suggested by
Theorem~\ref{theorem:ical-labelcomplexity}, which states that the
label complexity improvement may scale with the error of the optimal
hypothesis.  Also, the results from MNIST tasks suggest that the
active learner may require an initial random sampling phase during
which it is equivalent to the passive learner, and the advantage
manifests itself after this phase.  This again is consistent with the
analysis (also see~\cite{Han07}), as the disagreement coefficient can
be large at initial scales, yet much smaller as the number of
(unlabeled) data increases and the scale becomes finer.

\section{Conclusion}

This paper provides a new active learning algorithm based on error
minimization oracles, a departure from the version space approach
adopted by previous works.  The algorithm we introduce here motivates
computationally tractable and effective methods for active learning
with many classifier training algorithms.  The overall algorithmic template
applies to any training algorithm that (i) operates by approximate error
minimization and (ii) for which the cost of switching a class prediction (as
measured by example errors) can be estimated.  Furthermore, although
these properties might only hold in an approximate or heuristic sense,
the created active learning algorithm will be ``safe'' in the sense
that it will eventually converge to the same solution as a passive
supervised learning algorithm.  Consequently, we believe this approach
can be widely used to reduce the cost of labeling in situations where
labeling is expensive.

Recent theoretical work on active learning has focused on improving
rates of convergence.  However, in some applications, it may be
desirable to improve performance at much smaller sample sizes, perhaps
even at the cost of improved rates as long as consistency is ensured.
Importance sampling and weighting techniques like those analyzed in
this work may be useful for developing more aggressive strategies with
such properties.

\subsubsection*{References} 
{\def\section*#1{}\small \bibliography{oracle} \bibliographystyle{alpha}}

\vfill

\pagebreak

\appendix

\section{Proof of Deviation Bound for Importance Weighted Estimators}

The techniques here are mostly developed in~\cite{Zha05}; for completeness,
we detail the proofs for our particular application.
The first two lemmas establish a basic bound in terms of conditional moment
generating functions.
\begin{lemma} \label{lemma:mgf}
For all $n \geq 1$ and all functionals $\Xi_i := \xi_i(Z_{1:i})$,
\begin{equation*}
\E\left[
\exp\left( \sum_{i=1}^n \Xi_i
- \sum_{i=1}^n \ln \E_i[\exp(\Xi_i)]
\right)
\right]
\ = \
1
.
\end{equation*}
\end{lemma}
\begin{proof}
A straightforward induction on $n$.
\end{proof}
\begin{lemma} \label{lemma:mgf-deviation}
For all $t \geq 0$, $\lambda \in \R$, $n \geq 1$, and functionals $\Xi_i :=
\xi_i(Z_{1:i})$,
\begin{equation*}
\Pr\left( \lambda \sum_{i=1}^n \Xi_i - \sum_{i=1}^n \ln \E_i[\exp(\lambda
\Xi_i)] \geq t \right) \leq e^{-t}
.
\end{equation*}
\end{lemma}
\begin{proof}
The claim follows by Markov's inequality and Lemma~\ref{lemma:mgf}
(replacing $\Xi_i$ with $\lambda \Xi_i$).
\end{proof}

In order to specialize Lemma~\ref{lemma:mgf-deviation} for our purposes, we
first analyze the conditional moment generating function of $W_i -
\E_i[W_i]$.
\begin{lemma} \label{lemma:iw-mgf}
If $0 < \lambda < 3 P_i$, then
\begin{equation*}
\ln \E_i[\exp(\lambda(W_i - \E_i[W_i]))]
\ \leq \
\frac{1}{P_i} \cdot \frac{\lambda^2}{2(1-\lambda/(3P_i))}
.
\end{equation*}
If $\E_i[W_i] = 0$, then
\begin{equation*}
\ln \E_i[\exp(\lambda(W_i - \E_i[W_i]))]
\ = \
0
.
\end{equation*}
\end{lemma}
\begin{proof}
Let $g(x) := (\exp(x)-x-1)/x^2$ for $x \neq 0$, so $\exp(x) = 1+x+x^2 \cdot
g(x)$.
Note that $g(x)$ is non-decreasing.
Thus,
\begin{eqnarray*}
\lefteqn{
\E_i\left[
\exp(\lambda(W_i-\E_i[W_i]))
\right]
}
\\
& = &
\E_i\left[
1 + \lambda(W_i-\E_i[W_i]) + \lambda^2(W_i-\E_i[W_i])^2
\cdot g(\lambda(W_i-\E_i[W_i]))
\right]
\\
& = & 1 + \lambda^2 \cdot \E_i\left[
(W_i-\E_i[W_i])^2 \cdot g(\lambda(W_i-\E_i[W_i]))
\right]
\\
& \leq & 1 + \lambda^2 \cdot \E_i\left[
(W_i-\E_i[W_i])^2 \cdot g(\lambda/P_i)
\right]
\\
& = & 1 + \lambda^2 \cdot \E_i\left[
(W_i-\E_i[W_i])^2 \right] \cdot g(\lambda/P_i)
\\
& \leq & 1 + (\lambda^2/P_i) \cdot g(\lambda/P_i)
\end{eqnarray*}
where the first inequality follows from the range bound $|W_i| \leq 1/P_i$
and the second follows from variance bound $\E_i[(W_i-\E_i[W_i])^2] \leq
1/P_i$.
Now the first claim follows from the definition of $g(x)$, the facts
$\exp(x) - x - 1 \leq x^2/(2(1-x/3))$ for $0 \leq x < 3$ and $\ln(1+x) \leq
x$.

The second claim is immediate from the definition of $W_i$ and the fact
$\E_i[W_i] = f(X_i,Y_i)$.
\end{proof}

We now combine Lemma~\ref{lemma:iw-mgf} and Lemma~\ref{lemma:mgf-deviation}
to bound the deviation of the importance weighted estimator $\wh
f(Z_{1:n})$ from $(1/n) \sum_{i=1}^n \E_i[W_i]$.
\begin{lemma} \label{lemma:iw-deviation}
Pick any $t \geq 0$, $n \geq 1$, and $p_{min} > 0$,
and let $E$ be the (joint) event
\begin{eqnarray*}
\lefteqn{
\frac1n \sum_{i=1}^n W_i
- \frac1n \sum_{i=1}^n \E_i[W_i]
\ \geq \
\sqrt{\frac{1}{p_{min}} \cdot \frac{2t}{n}}
+ \frac{1}{p_{min}} \cdot \frac{t}{3n}
}
\\
& \text{and} &
\min\{P_i : 1 \leq i \leq n \ \wedge \ \E_i[W_i] \neq 0 \}
\ \geq \
p_{min}
.
\end{eqnarray*}
Then $\Pr(E) \leq e^{-t}$.
\end{lemma}
\begin{proof}
With foresight, let
\begin{equation*}
\lambda
\ := \
3p_{min} \cdot \frac{\sqrt{\frac{1}{3p_{min}} \cdot \frac{2t}{3n}}}
{1 + \sqrt{\frac{1}{3p_{min}} \cdot \frac{2t}{3n}}}
.
\end{equation*}
Note that $0 < \lambda < 3p_{min}$.
By Lemma~\ref{lemma:iw-mgf} and the choice of $\lambda$, we have that if
$\min\{ P_i : 1 \leq i \leq n \ \wedge \ \E_i[W_i] \neq 0 \} \geq p_{min}$,
then
\begin{equation} \label{eq:iw-term1}
\frac1{n\lambda} \cdot \sum_{i=1}^n \ln \E_i[\exp(\lambda (W_i-\E_i[W_i]))] 
\ \leq \
\frac{1}{p_{min}} \cdot \frac{\lambda}{2(1-\lambda/(3p_{min}))}
\ = \
\sqrt{\frac{1}{p_{min}} \cdot \frac{t}{2n}}
\end{equation}
and
\begin{equation} \label{eq:iw-term2}
\frac{t}{n\lambda}
\ = \
\sqrt{\frac{1}{p_{min}} \cdot \frac{t}{2n}}
+
\frac{1}{p_{min}} \cdot \frac{t}{3n}
.
\end{equation}
Let $E'$ be the event that
\begin{equation*}
\frac1n \cdot \sum_{i=1}^n (W_i-\E_i[W_i]) -
\frac1{n\lambda} \cdot \sum_{i=1}^n \ln \E_i[\exp(\lambda (W_i-\E_i[W_i]))]
\ \geq \
\frac{t}{n\lambda}
\end{equation*}
and let $E''$ be the event $\min\{ P_i : 1 \leq i \leq n \ \wedge \
\E_i[W_i] \neq 0 \} \ \geq \ p_{min}$.
Together, Eq.~\eqref{eq:iw-term1} and Eq.~\eqref{eq:iw-term2} imply $E
\subseteq E' \cap E''$.
And of course, $E' \cap E'' \subseteq E'$, so
$\Pr(E) \leq \Pr(E' \cap E'') \leq \Pr(E') \leq e^{-t}$
by Lemma~\ref{lemma:mgf-deviation}.
\end{proof}

To do away with the joint event in Lemma~\ref{lemma:iw-deviation}, we use
the standard trick of taking a union bound over a geometric sequence of
possible values for $p_{min}$.
\begin{lemma} \label{lemma:iw-deviation2}
Pick any $t \geq 0$ and $n \geq 1$.
Assume $1 \leq 1 / P_i \leq r_{max}$ for all $1 \leq i \leq n$, and let
$R_n := 1 / \min\{P_i : 1 \leq i \leq n \ \wedge \ \E_i[W_i] \neq 0 \} \cup
\{ 1 \}$.
We have
\begin{equation*}
\Pr\left(
\left|
\frac1n \sum_{i=1}^n W_i
- \frac1n \sum_{i=1}^n \E_i[W_i]
\right|
\ \geq \
\sqrt{\frac{2R_nt}{n}}
+ \frac{R_nt}{3n}
\right)
\ \leq \
2(2+\log_2 r_{max}) e^{-t/2}
.
\end{equation*}
\end{lemma}
\begin{proof}
The assumption on $P_i$ implies $1 \leq R_n \leq r_{max}$.
Let $r_j := 2^j$ for $-1 \leq j \leq m := \lceil\log_2 r_{max} \rceil$.
Then
\begin{eqnarray*}
\lefteqn{
\Pr\left(
\frac1n \sum_{i=1}^n W_i
- \frac1n \sum_{i=1}^n \E_i[W_i]
\ \geq \
\sqrt{\frac{2R_nt}{n}}
+ \frac{R_nt}{3n}
\right)
}
\\
& = &
\sum_{j=0}^m
\Pr\left(
\frac1n \sum_{i=1}^n W_i
- \frac1n \sum_{i=1}^n \E_i[W_i]
\ \geq \
\sqrt{\frac{2R_nt}{n}}
+ \frac{R_nt}{3n}
\ \wedge \
r_{j-1} < R_n \leq r_j
\right)
\\
& \leq &
\sum_{j=0}^m
\Pr\left(
\frac1n \sum_{i=1}^n W_i
- \frac1n \sum_{i=1}^n \E_i[W_i]
\ \geq \
\sqrt{\frac{2r_{j-1}t}{n}}
+ \frac{r_{j-1}t}{3n}
\ \wedge \
R_n \leq r_j
\right)
\\
& = &
\sum_{j=0}^m
\Pr\left(
\frac1n \sum_{i=1}^n W_i
- \frac1n \sum_{i=1}^n \E_i[W_i]
\ \geq \
\sqrt{\frac{2r_j(t/2)}{n}}
+ \frac{r_j(t/2)}{3n}
\ \wedge \
R_n \leq r_j
\right)
\\
& \leq &
(2+\log_2 r_{max}) e^{-t/2}
\end{eqnarray*}
where the last inequality follows from Lemma~\ref{lemma:iw-deviation}.
Replacing $W_i$ with $-W_i$ bounds the probability of deviations in the
other direction in exactly the same way.
The claim then follows by the union bound.
\end{proof}

\begin{proof}[Proof of Theorem~\ref{theorem:iw-deviation}]
By Hoeffding's inequality and the fact $|f(X_i,Y_i)| \leq 1$, we have
\begin{equation*}
\Pr\left(
\left|
\frac1n \sum_{i=1}^n f(X_i,Y_i)
-
\E[f(X,Y)]
\right|
\ \geq \
\sqrt{\frac{2t}{n}}
\right)
\ \leq \
2e^{-t/2}
.
\end{equation*}
Since $\E_i[W_i] = f(X_i,Y_i)$, the claim follows by combining this and
Lemma~\ref{lemma:iw-deviation2} with the triangle inequality and the union
bound.
\end{proof}

\section{Remaining Proofs}

In this section, we use the notation $\veps_k := C_0 \log(k+1) / k$.

\subsection{Proof of Lemma~\ref{lemma:ical-crude-bound}}

By induction on $n$.
Trivial for $n = 1$ (since $p(\text{empty sequence},x) = 1$ for all $x \in
\X$), so now fix any $n \geq 2$ and assume as the inductive hypothesis
$p_{n-1} = p(z_{1:n-2},x) \geq 1/(n-1)^{n-1}$ for all $(z_{1:n-2},x) \in
(\X\times\Y\times\{0,1\})^{n-2} \times \X$.
Fix any $(z_{1:n-1},x) \in (\X\times\Y\times\{0,1\})^{n-1} \times \X$, and
consider the error difference $g_n := \err(h_n',z_{1:n-1}) -
\err(h_n,z_{1:n-1})$ used to determine $p_n := p(z_{1:n-1},x)$.
We only have to consider the case $g_n > \sqrt{\veps_{n-1}} + \veps_{n-1}$.
By the inductive hypothesis and triangle inequality, we have $g_n \leq 2
(n-1)^{n-1}$.
Solving the quadratic in Eq.~\eqref{eq:ical-quadratic} implies
\begin{align*}
\sqrt{p_n}
& =
\frac{c_1 \cdot \sqrt{\veps_{n-1}} + \sqrt{c_1^2 \cdot \veps_{n-1} 
+ 4 \cdot \left(g_n + (c_1-1) \cdot \sqrt{\veps_{n-1}} + (c_2-1) \cdot
\veps_{n-1} \right) \cdot c_2 \cdot \veps_{n-1}}}
{2\left(g_n + (c_1-1) \cdot \sqrt{\veps_{n-1}} + (c_2-1) \cdot
\veps_{n-1} \right)}
\\
& >
\frac{\sqrt{4 \cdot \left(g_n + (c_1-1) \cdot \sqrt{\veps_{n-1}} + (c_2-1)
\cdot \veps_{n-1} \right) \cdot c_2 \cdot \veps_{n-1}}}
{2\left(g_n + (c_1-1) \cdot \sqrt{\veps_{n-1}} + (c_2-1) \cdot
\veps_{n-1} \right)}
\quad \text{(dropping terms)}
\\
& =
\sqrt{\frac{c_2 \cdot \veps_{n-1}}{g_n + (c_1-1) \cdot \sqrt{\veps_{n-1}} + (c_2-1) \cdot
\veps_{n-1}}}
\\
& \geq
\sqrt{\frac{c_2 \cdot \veps_{n-1}}{g_n + (c_1-1) \cdot \sqrt{\veps_{n-1}} +
(c_1-1) \cdot \veps_{n-1}}}
\quad \text{(since $c_2 \leq c_1$)}
\\
& \geq
\sqrt{\frac{c_2 \cdot \veps_{n-1}}{c_1 \cdot g_n}}
\quad \text{(since $g_n > \sqrt{\veps_{n-1}} + \veps_{n-1}$)}
\\
& =
\sqrt{\frac{c_2 \cdot C_0 \log n}{c_1 \cdot (n-1) \cdot g_n}}
\\
& \geq
\sqrt{\frac{c_2 \cdot C_0 \log n}{2c_1 \cdot (n-1) \cdot (n-1)^{n-1}}}
\quad \text{(inductive hypothesis)}
\\
& >
\sqrt{\frac{1}{e (n-1)^n}}
\quad \text{(since $C_0 \geq 2$, $n \geq 2$, and $(c_2 \cdot C_0 \log
2)/(2c_1) > 1/e$)}
\\
& \geq
\sqrt{\frac{1}{n^n}}
\quad \text{(since $(n/(n-1))^n \geq e$)}
\end{align*}
as required.
\qed

\subsection{Proof of Theorem~\ref{theorem:ical-consistency}}

We condition on the $1-\delta$ probability event that the deviation bounds
from Lemma~\ref{lemma:ical-deviation} hold (also using
Lemma~\ref{lemma:ical-crude-bound}).
The proof now proceeds by induction on $n$.
The claim is trivially true for $n = 1$.
Now pick any $n \geq 2$ and assume as the (strong) inductive hypothesis
that
\begin{equation} \label{eq:consistency-ind-hyp}
0
\ \leq \
\err(h_k) - \err(h^*)
\ \leq \
\err(h_k,Z_{1:k-1}) - \err(h^*,Z_{1:k-1})
+ \sqrt{2\veps_{k-1}} + 2\veps_{k-1}
\end{equation}
for all $1 \leq k \leq n - 1$.
We need to show Eq.~\eqref{eq:consistency-ind-hyp} holds for $k = n$.

Let $P_{min} := \min \{ P_i : 1 \leq i \leq n-1 \ \wedge \ h_n(X_i) \neq
h^*(X_i) \} \cup \{ 1 \}$.
If $P_{min} \geq 1/2$, then Eq.~\eqref{eq:ical-deviation} implies that
Eq.~\eqref{eq:consistency-ind-hyp} holds for $k = n$ as needed.
So assume for sake of contradiction that $P_{min} < 1/2$, and let $n_0 :=
\max\{ i \leq n - 1 : P_i = P_{min} \ \wedge \ h_n(X_i) \neq h^*(X_i) \}$.
By definition of $P_{n_0}$, we have
\begin{equation*}
\err(h_{n_0}', Z_{1:n_0-1}) - \err(h_{n_0}, Z_{1:n_0-1})
=
\left( \frac{c_1}{\sqrt{P_{min}}} - c_1 + 1 \right)
\sqrt{\veps_{n_0-1}}
+ \left( \frac{c_2}{P_{min}} - c_2 + 1 \right)
\veps_{n_0-1}
.
\end{equation*}
Using this fact together with the inductive hypothesis, we have
\begin{align}
\lefteqn{\err(h_{n_0}', Z_{1:n_0-1}) - \err(h^*, Z_{1:n_0-1})}
\nonumber
\\
& =
\err(h_{n_0}', Z_{1:n_0-1}) - \err(h_{n_0}, Z_{1:n_0-1})
+ \err(h_{n_0}, Z_{1:n_0-1}) - \err(h^*, Z_{1:n_0-1})
\nonumber
\\
& \geq
\left( \frac{c_1}{\sqrt{P_{min}}} - c_1 + 1 \right) \cdot
\sqrt{\veps_{n_0-1}}
+ \left( \frac{c_2}{P_{min}} - c_2 + 1 \right) \cdot \veps_{n_0-1}
- \sqrt{2\veps_{n_0-1}} - 2\veps_{n_0-1}
\nonumber
\\
& =
\left( \frac{c_1}{\sqrt{P_{min}}} - c_1 + 1 - \sqrt{2} \right) \cdot
\sqrt{\veps_{n_0-1}}
+ \left( \frac{c_2}{P_{min}} - c_2 - 1 \right) \cdot \veps_{n_0-1}
\quad \quad
\label{eq:ind-step-gap}
.
\end{align}
We use the assumption $P_{min} < 1/2$ to lower bound the righthand side
to get the inequality
\begin{equation*}
\err(h_{n_0}', Z_{1:n_0-1}) - \err(h^*, Z_{1:n_0-1})
>
(c_1 - 1) \cdot (\sqrt{2} - 1) \cdot \sqrt{\veps_{n_0-1}}
+ (c_2 - 1) \cdot \veps_{n_0-1}
> 0
.
\end{equation*}
which implies $\err(h_{n_0}',Z_{1:n_0-1}) > \err(h^*,Z_{1:n_0-1})$.
Since $h_{n_0}'$ minimizes $\err(h,Z_{1:n_0-1})$ among hypotheses $h \in
\H$ that disagree with $h_{n_0}$ on $X_{n_0}$, it must be that $h^*$ agrees
with $h_{n_0}$ on $X_{n_0}$.
By transitivity and the definition of $n_0$, we conclude that $h_n(X_{n_0})
= h_{n_0}'(X_{n_0})$; so $\err(h_n,Z_{1:n_0-1}) \geq
\err(h_{n_0}',Z_{1:n_0-1})$.
Then
\begin{align}
\lefteqn{\err(h_n,Z_{1:n-1}) - \err(h^*,Z_{1:n-1})}
\nonumber
\\
& \geq
\err(h_n) - \err(h^*)
- \sqrt{\frac{1}{P_{min}} \cdot \veps_{n-1}}
- \frac{1}{P_{min}} \cdot \veps_{n-1}
\nonumber
\\
& \geq
\err(h_n,Z_{1:n_0-1}) - \err(h^*,Z_{1:n_0-1})
- 2 \cdot \sqrt{\frac{1}{P_{min}} \cdot \veps_{n_0-1}}
- 2 \cdot \frac{1}{P_{min}} \cdot \veps_{n_0-1}
\nonumber
\\
& \geq
\left( \frac{c_1-2}{\sqrt{P_{min}}} - c_1 + 1 - \sqrt{2} \right) \cdot
\sqrt{\veps_{n_0-1}}
+ \left( \frac{c_2-2}{P_{min}} - c_2 - 1 \right) \cdot
\veps_{n_0-1}
\nonumber
\\
& >
\left( (c_1 - 1) \cdot (\sqrt{2} - 1) - 2\sqrt{2} \right)
\cdot \sqrt{\veps_{n_0-1}}
+ \left( c_2 - 5 \right) \cdot \veps_{n_0-1}
\nonumber
\end{align}
where Eq.~\eqref{eq:ical-deviation} is used in the first two inequalities,
Eq.~\eqref{eq:ind-step-gap} and the fact $\err(h_n,Z_{1:n_0-1}) \geq
\err(h_{n_0}',Z_{1:n_0-1})$ are used in the third inequality, and the fact
$P_{min} < 1/2$ is used in the last inequality.
This final quantity is non-negative, so we have the contradiction
$\err(h_n,Z_{1:n-1}) > \err(h^*,Z_{1:n-1})$.
\qed

\subsection{Proof of Lemma~\ref{lemma:ical-query}}

First, we establish a property of the query probabilities that relates
error deviations (via $P_{min}$) to empirical error differences (via
$P_n$).
Both quantities play essential roles in bounding the label complexity
through the disagreement metric structure around $h^*$.
\begin{lemma} \label{lemma:ical-pmin-bound}
Assume the bounds from Eq.~\eqref{eq:ical-deviation} hold for all $h \in
\H$ and $n \geq 1$.
For any $n \geq 1$, we have $P_n \leq c_3 \cdot P_{min}$, where $P_{min} :=
\min ( \{ P_i : 1 \leq i \leq n-1 \ \wedge \ h(X_i) \neq h^*(X_i) \}
\cup \{ 1 \} )$
and
\begin{equation} \label{eq:defn-h}
h := \left\{ \begin{array}{ll} h_n &
\text{if $h_n$ disagrees with $h^*$ on $X_n$}
\\
h_n' &
\text{if $h_n'$ disagrees with $h^*$ on $X_n$.}
\end{array} \right.
\end{equation}
\end{lemma}
\begin{proof}
We can assume $P_{min} < 1/c_3$, since otherwise the claim is trivial.
Pick any $n_0 \leq n-1$ such that $h(X_{n_0}) \neq h^*(X_{n_0})$ and
$P_{n_0} = P_{min}$ (such an $n_0$ is guaranteed to exist given the above
assumption).
We now proceed as in the proof of Theorem~\ref{theorem:ical-consistency}.
We first show a lower bound on $\err(h,Z_{1:n_0-1}) -
\err(h^*,Z_{1:n_0-1})$.
Note that
\begin{align}
\lefteqn{\err(h_{n_0}',Z_{1:n_0-1}) - \err(h^*,Z_{1:n_0-1})}
\nonumber
\\
& =
\err(h_{n_0}',Z_{1:n_0-1}) - \err(h_{n_0},Z_{1:n_0-1})
+\err(h_{n_0},Z_{1:n_0-1}) - \err(h^*,Z_{1:n_0-1})
\nonumber
\\
& \geq \left( \frac{c_1}{\sqrt{P_{min}}} - c_1 + 1 \right) \cdot
\sqrt{\veps_{n_0-1}}
+ \left( \frac{c_2}{P_{min}} - c_2 + 1 \right) \cdot \veps_{n_0-1}
- \sqrt{2\veps_{n_0-1}} - 2\veps_{n_0-1}
\nonumber
\\
& =
\left( \frac{c_1}{\sqrt{P_{min}}} - c_1 + 1 - \sqrt{2} \right) \cdot
\sqrt{\veps_{n_0-1}}
+ \left( \frac{c_2}{P_{min}} - c_2 - 1 \right) \cdot
\veps_{n_0-1}
\label{eq:past-gap}
\end{align}
where the inequality follows from Theorem~\ref{theorem:ical-consistency}.
The righthand side is positive, so $h^*$ must disagree with $h_{n_0}'$ on
$X_{n_0}$.
By transitivity (recalling that $h(X_{n_0}) \neq h^*(X_{n_0})$), $h$ must
agree with $h_{n_0}'$ on $X_{n_0}$.
Therefore $\err(h,Z_{1:n_0-1}) - \err(h_{n_0}',Z_{1:n_0-1}) \geq 0$, so the
inequality in Eq.~\eqref{eq:past-gap} holds with $h$ in place of $h_{n_0}'$
on the lefthand side.

Now $\err(h,Z_{1:n-1}) - \err(h^*,Z_{1:n-1})$ is related to
$\err(h,Z_{1:n_0-1}) - \err(h^*,Z_{1:n_0-1})$ through $\err(h)-\err(h^*)$
using the deviation bound from Eq.~\eqref{eq:ical-deviation} (as well as
the fact $\veps_{n_0-1} \geq \veps_{n-1}$):
\begin{align}
\lefteqn{\err(h,Z_{1:n-1}) - \err(h^*,Z_{1:n-1})}
\nonumber
\\
& \geq
\err(h,Z_{1:n_0-1}) - \err(h^*,Z_{1:n_0-1})
- 2 \cdot \sqrt{\frac{1}{P_{min}} \cdot \veps_{n_0-1}}
- 2 \cdot \frac{1}{P_{min}} \cdot \veps_{n_0-1}
\nonumber
\\
& \geq
\left( \frac{c_1-2}{\sqrt{P_{min}}} - c_1 + 1 - \sqrt{2} \right) \cdot
\sqrt{\veps_{n-1}}
+ \left( \frac{c_2-2}{P_{min}} - c_2 - 1 \right) \cdot
\veps_{n-1}
\ > \ 0
\label{eq:current-gap-lb}
.
\end{align}
If $h = h_n$, then $\err(h,Z_{1:n-1}) - \err(h^*,Z_{1:n-1}) =
\err(h_n,Z_{1:n-1}) - \err(h^*,Z_{1:n-1}) \leq 0$ by the minimality of
$\err(h_n,Z_{1:n-1})$; this contradicts Eq.~\eqref{eq:current-gap-lb}.
Therefore it must be that $h = h_n'$.
In this case,
\begin{align}
\err(h,Z_{1:n-1}) - \err(h^*,Z_{1:n-1})
& \leq
\err(h_n',Z_{1:n-1}) - \err(h_n,Z_{1:n-1})
\nonumber
\\
& =
\left( \frac{c_1}{\sqrt{P_n}} - c_1 + 1 \right) \cdot
\sqrt{\veps_{n-1}}
+ \left( \frac{c_2}{P_n} - c_2 + 1 \right) \cdot
\veps_{n-1}
\label{eq:current-gap}
\end{align}
where the inequality follows from the minimality of $\err(h_n,Z_{1:n-1})$,
and the subsequent step follows from the definition of $P_n$.
Combining the lower bound in Eq.~\eqref{eq:current-gap-lb} and the upper
bound in Eq.~\eqref{eq:current-gap} implies that
\begin{equation*}
\frac{c_1}{\sqrt{P_n}} \cdot \sqrt{\veps_{n-1}}
+ \frac{c_2}{P_n} \cdot \veps_{n-1}
\geq
\left( \frac{c_1-2}{\sqrt{P_{min}}} - \sqrt{2} \right) \cdot
\sqrt{\veps_{n-1}}
+ \left( \frac{c_2-2}{P_{min}} - 2 \right) \cdot
\veps_{n-1}
.
\end{equation*}
It is easily checked that this implies $P_n \leq c_3 \cdot P_{min}$.
\end{proof}

\begin{proof}[Proof of Lemma~\ref{lemma:ical-query}]
Define $h$ as in Eq.~\eqref{eq:defn-h}.
By Lemma~\ref{lemma:ical-pmin-bound}, we have
$\min (\{ P_i : 1 \leq i \leq n-1 \ \wedge \ h(X_i) \neq h^*(X_i) \} \cup
\{ 1 \}) \geq P_n / c_3$.
We first show that
\begin{align}
\err(h) - \err(h^*)
& \leq
\err(h,Z_{1:n-1}) - \err(h^*,Z_{1:n-1})
+ \sqrt{\frac{c_3}{P_n} \cdot \veps_{n-1}}
+ \frac{c_3}{P_n} \cdot \veps_{n-1}
\nonumber
\\
& \leq
\sqrt{\frac{c_4}{P_n}} \cdot \sqrt{\veps_{n-1}}
+ \frac{c_5}{P_n} \cdot \veps_{n-1}
\label{eq:query-quadratic}
.
\end{align}
The first inequality follows from Eq.~\eqref{eq:ical-deviation} and
Lemma~\ref{lemma:ical-pmin-bound}.
For the second inequality, we consider two cases depending on $h$.
If $h = h_n'$, then we bound $\err(h,Z_{1:n-1})-\err(h^*,Z_{1:n-1})$ from
above by $\err(h_n',Z_{1:n-1})-\err(h_n,Z_{1:n-1})$ (by definition of $h$
and minimality of $\err(h_n,Z_{1:n-1})$), and then simplify
\begin{align*}
\lefteqn{\err(h_n',,Z_{1:n-1}) - \err(h_n,Z_{1:n-1})
+ \sqrt{\frac{c_3}{P_n} \cdot \veps_{n-1}}
+ \frac{c_3}{P_n} \cdot \veps_{n-1}}
\\
& \leq
\left( \frac{c_1 + \sqrt{c_3}}{\sqrt{P_n}} - c_1 + 1 \right)
\cdot \sqrt{\veps_{n-1}}
+ \left( \frac{c_2 + c_3}{P_n} - c_2 + 1 \right)
\cdot \veps_{n-1}
\ \leq
\sqrt{\frac{c_4}{P_n}} \cdot \sqrt{\veps_{n-1}}
+ \frac{c_5}{P_n} \cdot \veps_{n-1}
\end{align*}
using the definition of $P_n$ and the facts $c_1 \geq 1$ and $c_2 \geq 1$.
If instead $h = h_n$, then we use the facts
$\err(h,Z_{1:n-1})-\err(h^*,Z_{1:n-1}) = \err(h_n,Z_{1:n-1}) -
\err(h^*,Z_{1:n-1}) \leq 0$ and $c_3 \leq \min\{c_4, c_5\}$.

If $\err(h) - \err(h^*) = \gamma > 0$, then solving the quadratic
inequality in Eq.~\eqref{eq:query-quadratic} for $P_n$ gives the bound
\begin{equation*}
P_n
\ \leq \
\min\left\{ 1, \
\frac32 \cdot \left( \frac{c_4}{\gamma^2} +
\frac{c_5}{\gamma} \right) \cdot \veps_{n-1}
\right\}
.
\end{equation*}
If $\err(h) - \err(h^*) \leq \bar\gamma$, then by the triangle inequality
we have
\begin{equation*}
\Pr(h^*(X) \neq h(X))
\ \leq \
\err(h^*) + \err(h)
\ \leq \
2\err(h^*) + \bar\gamma
\end{equation*}
which in turn implies $X_n \in \DIS(h^*,2\err(h^*) + \bar\gamma)$.
Note that $\Pr(X_n \in \DIS(h^*,2\err(h^*)+\bar\gamma)) \leq \theta \cdot
(2\err(h^*) + \bar\gamma)$ by definition of $\theta$, so
$\Pr(\err(h)-\err(h^*) \leq \bar\gamma) \leq \theta \cdot (2\err(h^*) +
\bar\gamma)$.

Let $f(\gamma) := \partial \Pr(\err(h) - \err(h^*) \leq \gamma) / \partial
\gamma$ be the probability density (mass) function of the error difference
$\err(h) - \err(h^*)$; note that this error difference is a function of
$(Z_{1:n-1},X_n)$.
We compute the expected value of $Q_n$ by conditioning on
$\err(h)-\err(h^*)$ and integrating (an upper bound on)
$\E[Q_n|\err(h)-\err(h^*) = \gamma]$ with respect to $f(\gamma)$.

Let $\gamma_0 > 0$ be the positive solution to
$1.5(c_4/\gamma^2+c_5/\gamma)\veps_{n-1} = 1$.
It can be checked that $\gamma_0 > \sqrt{1.5 c_4 \veps_{n-1}}$.
We have
\begin{align*}
\E[Q_n]
& =
\E[\E[Q_n|Z_{1:n-1},X_n]]
\qquad \text{(the outer expectation is over $(Z_{1:n-1},X_n)$)}
\\
& =
\int_0^1
\left( \frac{\partial}{\partial \gamma} \Pr(\err(h)-\err(h^*) \leq \gamma) \right)
\cdot \E[Q_n | \err(h)-\err(h^*) = \gamma]
\cdot d\gamma
\\
& \leq
\int_0^1
\left( \frac{\partial}{\partial \gamma} \Pr(\err(h)-\err(h^*) \leq \gamma)
\right)
\cdot \min\left\{ 1, \ \frac32 \cdot \left( \frac{c_4}{\gamma^2} +
\frac{c_5}{\gamma} \right) \cdot \veps_{n-1} \right\}
\cdot d\gamma
\\
& \leq
\frac32 \cdot (c_4+c_5) \cdot \veps_{n-1} \cdot \Pr(\err(h)-\err(h^*) \leq
1)
\\
& \quad {}
- \int_0^1
\left( \frac{\partial}{\partial \gamma}
\min\left\{ 1, \ \frac32 \cdot \left( \frac{c_4}{\gamma^2} + \frac{c_5}{\gamma}
\right) \cdot \veps_{n-1} \right\} \right)
\cdot \Pr(\err(h)-\err(h^*) \leq \gamma)
\cdot d\gamma
\\
& \leq
\frac32 \cdot (c_4+c_5) \cdot \veps_{n-1}
+ \int_{\gamma_0}^1
\frac32 \cdot \left( \frac{2c_4}{\gamma^3} + \frac{c_5}{\gamma^2}
\right) \cdot \veps_{n-1}
\cdot \theta \cdot (2\err(h^*) + \gamma)
\cdot d\gamma
\\
& =
\frac32 \cdot (c_4+c_5) \cdot \veps_{n-1}
+ \theta \cdot 2 \err(h^*) \cdot \frac32 \cdot \left(
c_4 \left( \frac1{\gamma_0^2} - 1 \right)
+ c_5 \left( \frac1{\gamma_0} - 1 \right)
\right) \cdot \veps_{n-1}
\\
& \quad {}
+ \theta \cdot \frac32 \cdot \left(
2c_4 \left( \frac{1}{\gamma_0} - 1 \right)
+ c_5 \ln \frac1{\gamma_0}
\right) \cdot \veps_{n-1}
\\
& \leq
\frac32 \cdot (c_4+c_5) \cdot \veps_{n-1}
+ \theta \cdot 2\err(h^*)
+ \theta \cdot \sqrt{6c_4\veps_{n-1}}
+ \theta \cdot \frac{3c_5}{4} \cdot \veps_{n-1} \cdot
\ln \frac{1}{1.5c_4\veps_{n-1}}
\end{align*}
where the first inequality uses the bound on
$\E[Q_n|\err(h)-\err(h^*)=\gamma]$; the second inequality uses
integration-by-parts; the third inequality uses the fact that the integrand
from the previous line is $0$ for $0 \leq \gamma \leq \gamma_0$, as well as
the bound on $\Pr(\err(h)-\err(h^*) \leq \gamma)$; and the fourth
inequality uses the definition of $\gamma_0$.
\end{proof}

\subsection{Proof of Theorem~\ref{theorem:ical-labelcomplexity2}}

The theorem is a simple consequence of the following analogue of
Lemma~\ref{lemma:ical-query}.
\begin{lemma} \label{lemma:ical-query2}
Assume that for some value of $\kappa > 0$ and $0 < \alpha \leq 1$, the
condition in Eq.~\eqref{eq:tsybakov} holds for all $h \in \H$.
Assume the bounds from Eq.~\eqref{eq:ical-deviation} holds for all $h \in
\H$ and $n \geq 1$.
There is a constant $c_\alpha > 0$ such that the following holds.
For any $n \geq 1$,
\begin{equation*}
\E[Q_n]
\ \leq \
\theta \cdot \kappa \cdot c_\alpha
\cdot
\left( \frac{C_0 \log n}{n-1} \right)^{\alpha/2}
.
\end{equation*}
\end{lemma}
\begin{proof}
For the most part, the proof is the same as that of
Lemma~\ref{lemma:ical-query}.
The key difference is to use the noise condition in Eq.~\eqref{eq:tsybakov}
to directly bound $\Pr(h(X) \neq h^*(X)) \leq \kappa \cdot (\err(h) -
\err(h^*))^{\alpha}$, which in turn implies the bound $\Pr(\err(h) -
\err(h^*) \leq \gamma) \leq \theta \kappa \gamma^{\alpha}$.
As before, let $\gamma_0 > \sqrt{1.5 c_4 \veps_{n-1}}$ be the solution to
$1.5(c_4/\gamma^2+c_5/\gamma)\veps_{n-1} = 1$.
First consider the case $\alpha < 1$.
Then, the expectation of $Q_n$ can be bounded as
\begin{align*}
\E[Q_n]
& \leq
\frac32 \cdot (c_4+c_5) \cdot \veps_{n-1}
+ \int_{\gamma_0}^1
\frac32 \cdot \left( \frac{2c_4}{\gamma^3} + \frac{c_5}{\gamma^2}
\right) \cdot \veps_{n-1}
\cdot \Pr(\err(h)-\err(h^*)\leq\gamma)
\cdot d\gamma
\\
& \leq
\frac32 \cdot (c_4+c_5) \cdot \veps_{n-1}
+ \int_{\gamma_0}^1
\frac32 \cdot \left( \frac{2c_4}{\gamma^3} + \frac{c_5}{\gamma^2}
\right) \cdot \veps_{n-1}
\cdot \theta \kappa \gamma^{\alpha}
\cdot d\gamma
\\
& \leq
\frac32 \cdot (c_4+c_5) \cdot \veps_{n-1}
+ \theta \cdot \kappa \cdot \frac32 \cdot
\left( \frac{2c_4}{2-\alpha} \cdot \frac{1}{\gamma_0^{2-\alpha}}
+ \frac{c_5}{1-\alpha} \cdot \frac{1}{\gamma_0^{1-\alpha}}
\right)
\cdot \veps_{n-1}
.
\end{align*}
The case $\alpha = 1$ is handled similarly.
\end{proof}

\subsection{Analogue of Theorem~\ref{theorem:ical-consistency} under Low
Noise Conditions}

We first state a variant of Lemma~\ref{lemma:ical-deviation} that takes
into account the probability of disagreement between a hypothesis $h$
and the optimal hypothesis $h^*$.
\begin{lemma}
There exists an absolute constant $c > 0$ such that the following holds.
Pick any $\delta \in (0,1)$.
For all $n \geq 1$, let
\begin{equation*}
\veps_n := \frac{c \cdot \log((n+1)|\H|/\delta)}{n}
.
\end{equation*}
Let $(Z_1, Z_2, \ldots) \in (\X\times\Y\times\{0,1\})^*$ be the sequence of
random variables specified in Section~\ref{section:iwal} using a rejection
threshold $p:(\X\times\Y\times\{0,1\})^* \times \X \to [0,1]$ that
satisfies $p(z_{1:n},x) \geq 1/n^n$ for all $(z_{1:n},x) \in
(\X\times\Y\times\{0,1\})^n \times \X$ and all $n \geq 1$.

The following holds with probability at least $1-\delta$.
For all $n \geq 1$ and all $h \in \H$,
\begin{equation*}
|(\err(h,Z_{1:n}) - \err(h^*,Z_{1:n})) - (\err(h) - \err(h^*))|
\leq \sqrt{\frac{\Pr(h(X) \neq h^*(X))}{P_{min,n}(h)}\cdot \veps_n}
+ \frac{\veps_n}{P_{min,n}(h)}
\end{equation*}
where
$P_{min,n}(h)
\ = \ \min\{ P_i : 1 \leq i \leq n \ \wedge \ h(X_i) \neq h^*(X_i) \} \cup
\{1\}$.
\end{lemma}
\begin{proof}[Proof sketch]
The proof of this lemma follows along the same lines as that of
Lemma~\ref{lemma:ical-deviation}.
A key difference comes in Lemma~\ref{lemma:iw-deviation}:
the joint event is modified to also conjoin with
\[ \frac1n \sum_{i=1}^n \I(\E_i[f(X_i,Y_i)] \leq 0) \ \leq \ a
\]
for some fixed $a > 0$.
In the proof, the parameter $\lambda$ should be chosen as
\[ \lambda :=
3p_{min} \cdot \frac{\sqrt{\frac{1}{3p_{min}} \cdot \frac{2at}{3n}}}
{a + \sqrt{\frac{1}{3p_{min}} \cdot \frac{2at}{3n}}}
.
\]
Lemma~\ref{lemma:iw-deviation2} is modified to also take a union bound over
a sequence of possible values for $a$ (in fact, only $n+1$ different values
need to be considered).
Finally, instead of combining with Hoeffding's inequality, we use
Bernstein's inequality (or a multiplicative form of Chernoff's bound) so
the resulting bound (an analogue of Theorem~\ref{theorem:iw-deviation})
involves an empirical average inside the square-root term:
with probability at least $1-O(n \cdot \log_2 r_{max}) e^{-t/2}$,
\begin{equation*}
\left|
\frac1n \sum_{i=1}^n \frac{Q_i}{P_i} \cdot f(X_i,Y_i)
-
\E[f(X,Y)]
\right|
\ \leq \
O\left( \sqrt{\frac{R_nA_nt}{n}}
+ \frac{R_nt}{3n} \right)
\end{equation*}
where
\begin{equation*}
A_n := \frac1n \sum_{i=1}^n \I(f(X_i,Y_i) \neq 0)
.
\end{equation*}

Finally, we apply this deviation bound to obtain uniform error bounds over
all hypotheses $\H$ (a few extra steps are required to replace the
empirical quantity $A_n$ in the bound with a distributional quantity).
\end{proof}

Using the previous lemma, a modified version of
Theorem~\ref{theorem:ical-consistency} follows from essentially the same
proof.
We note that the quantity $C_1 := O(\log(|\H|/\delta))$ used here may
differ from $C_0$ by constant factors.
\begin{lemma}
The following holds with probability at least $1 - \delta$.
For any $n \geq 1$,
\begin{gather*}
0
\ \leq \
\err(h_n) - \err(h^*)
\ \leq \
\err(h_n,Z_{1:n-1}) - \err(h^*,Z_{1:n-1})
\\
+ \sqrt{\frac{2\Pr(h_n(X) \neq h^*(X)) C_1 \log n}{n-1}}
+ \frac{2C_1 \log n}{n-1}
.
\end{gather*}
This implies, for all $n \geq 1$,
\begin{equation*}
\err(h_n) \ \leq \ \err(h^*)
+ \sqrt{\frac{2\Pr(h_n(X) \neq h^*(X))C_1 \log n}{n-1}} + \frac{2C_0 \log
n}{n-1}
.
\end{equation*}
\end{lemma}
Finally, using the noise condition to bound $\Pr(h_n(X) \neq h^*(X)) \leq
\kappa \cdot (\err(h_n) - \err(h^*))^{\alpha}$, we obtain the final error
bound.
\begin{theorem}
The following holds with probability at least $1 - \delta$.
For any $n \geq 1$,
\begin{equation*}
\err(h_n) \ \leq \ \err(h^*)
+ c_{\kappa} \cdot \left( \frac{C_1 \log n}{n-1} \right)^{\frac1{2-\alpha}}
\end{equation*}
where $c_{\kappa}$ is a constant that depends only on $\kappa$.
\end{theorem}

\end{document}